\documentclass[12pt]{article}

\usepackage{times}
\usepackage{amsthm}
\usepackage{placeins}
\usepackage{graphicx}
\usepackage{amsmath}
\usepackage{epsfig}
\usepackage{color}
\usepackage{amssymb}

\usepackage{geometry}
\usepackage{titlesec}
\titleformat*{\section}{\fontsize{12}{15}\bfseries}
\titleformat*{\subsection}{\fontsize{12}{15}\bfseries}
\titlespacing\section{0pt}{12pt plus 4pt minus 2pt}{2pt plus 2pt minus 2pt}
\titlespacing\subsection{0pt}{12pt plus 4pt minus 2pt}{2pt plus 2pt minus 2pt}
\usepackage[utf8]{inputenc} \usepackage[T1]{fontenc}    \usepackage{hyperref}       \usepackage{url}            \usepackage{booktabs}       \usepackage{amsfonts}       \usepackage{nicefrac}       \usepackage{microtype}      \usepackage{txfonts}

\usepackage{natbib}

\newtheorem{theorem}{{\bf Theorem}}

\newtheorem{corollary}{{\bf Corollary}}
\newtheorem{lemma}{{\bf Lemma}}
\newtheorem{remark}{{\bf Remark}}

\usepackage{graphicx}
\usepackage{amsmath}
\usepackage{epsfig}

\makeatletter
\let\oldabs\abs
\def\abs{\@ifstar{\oldabs}{\oldabs*}}
\let\oldnorm\norm
\def\norm{\@ifstar{\oldnorm}{\oldnorm*}}
\makeatother

\usepackage{stmaryrd}

\newcommand{\I}[0]{\mathbb{I}}
\newcommand{\E}[0]{\mathbb{E}}
\newcommand{\R}[0]{\mathbb{R}}

\newcommand{\cR}{{\cal R}}

\def\R{\mathbb{R}}

\def\E{\mathbb E}

\def\I{{\mathbb I}}

\usepackage{mathtools}

\def\blackslug{\hbox{\hskip 1pt \vrule width 4pt height 8pt depth 1.5pt
\hskip 1pt}}
\def\qed{\quad\blackslug\lower 8.5pt\null\par}

\usepackage{cleveref}
\usepackage{caption}
\usepackage{multirow}
\usepackage[bottom]{footmisc}

\definecolor{gray}{RGB}{155,155,155}

\begin{document}

\begin{center}
{\fontsize{14}{0} \bf \selectfont
Weighted Empirical Risk Minimization:\\ Sample Selection Bias Correction based on Importance Sampling}

\bigskip
\medskip

{\fontsize{10}{0} \selectfont
    Robin Vogel$^{\text{*a, b}}$, Mastane Achab$^{\text{*a}}$, Stéphan Clémençon$^{\text{a}}$,
    Charles Tillier$^{\text{a}}$\\
    $^a$ LTCI, Télécom Paris\\
    19 Place Marguerite Perey, 91120 Palaiseau, France\\
    $^b$ IDEMIA,\\
    2 Place Samuel de Champlain, 92400 Courbevoie, France\\
    $^{*}$Corresponding Author: \texttt{firstname.lastname@telecom-paris.fr} \\
}
\end{center}
 
{ \par \centering \textbf{ABSTRACT} \par}

\noindent
We consider statistical learning problems, when the distribution $P'$ of the training observations $Z'_1,\; \ldots,\; Z'_n$ differs from the distribution $P$ involved in the risk one seeks to minimize (referred to as the \textit{test distribution}) but is still defined on the same measurable space as $P$ and dominates it. In the unrealistic case where the likelihood ratio $\Phi(z)=dP/dP'(z)$ is known, one may straightforwardly extends the Empirical Risk Minimization (ERM) approach to this specific \textit{transfer learning} setup using the same idea as that behind Importance Sampling, by minimizing a weighted version of the empirical risk functional computed from the 'biased' training data $Z'_i$ with weights $\Phi(Z'_i)$. Although the \textit{importance function} $\Phi(z)$ is generally unknown in practice, we show that, in various situations frequently encountered in practice, it takes a simple form and can be directly estimated from the $Z'_i$'s and some auxiliary information on the statistical population $P$. By means of linearization techniques, we then prove that the generalization capacity of the approach aforementioned is preserved when plugging the resulting estimates of the $\Phi(Z'_i)$'s into the  weighted empirical risk. Beyond these theoretical guarantees, numerical results provide strong empirical evidence of the relevance of the approach promoted in this article.

\bigskip \noindent
\textbf{Keywords:} 
Statistical Learning Theory, Importance Sampling, Transfer Learning.
 
\section{Introduction}

Prediction problems are of major importance in statistical learning. The main paradigm of predictive learning is \textit{Empirical Risk Minimization} (ERM in abbreviated form), see \textit{e.g.} \cite{Dev1}. In the standard setup, $Z$ is a random variable (r.v. in short) that takes its values in a feature space $\mathcal{Z}$ with distribution $P$, $\Theta$ is a parameter space and $\ell:\Theta \times \mathcal{Z}\rightarrow \mathbb{R}_+$ is a (measurable) loss function. The risk is then defined by: $\forall \theta\in \Theta$,
\begin{equation}\label{eq:risk}
\cR_P(\theta)=\mathbb{E}_P\left[ \ell(\theta, Z) \right],
\end{equation}
and more generally for any measure $Q$ on $\mathcal{Z}$: $\cR_Q(\theta)=\int_\mathcal{Z} \ell (\theta, z) dQ(z)$.
In most practical situations, the distribution $P$ involved in the definition of the risk is unknown and learning is based on the sole observation of an independent and identically distributed (i.i.d.) sample $Z_1,\; \ldots,\; Z_n$ drawn from $P$ and the risk \eqref{eq:risk} must be replaced by an empirical counterpart (or a possibly smoothed/penalized version of it), typically:
\begin{equation}\label{eq:emp_risk}
\widehat{\cR}_P(\theta)=\frac{1}{n}\sum_{i=1}^n \ell(\theta, Z_i)=\cR_{\widehat{P}_n}(\theta),
\end{equation}
where $\widehat{P}_n=(1/n)\sum_{i=1}^n\delta_{Z_i}$ is the empirical measure of $P$ and $\delta_z$ denotes the Dirac measure at any point $z$. With the design of successful algorithms such as neural networks, support vector machines or boosting methods to perform ERM,  the practice of predictive learning has recently received a significant attention and is now supported by a sound theory based on results in empirical process theory. The performance of minimizers of \eqref{eq:emp_risk} can be indeed studied by means of concentration inequalities, quantifying the fluctuations of the maximal deviations $\sup_{\theta\in \Theta}\vert \widehat{\cR}_P(\theta)- \cR_P(\theta) \vert$ under various complexity assumptions for the functional class $\mathcal{F}=\{\ell(\theta,\; \cdot):\; \theta\in \Theta\}$ (\textit{e.g.} {\sc VC} dimension, metric entropies, Rademacher averages), see \cite{boucheron2013concentration} for instance. Although, in the Big Data era, the availability of massive digitized information to train predictive rules is an undeniable opportunity for the widespread deployment of machine-learning solutions, the poor control of the data acquisition process one is confronted with in many applications puts practicioners at risk of jeopardizing the generalization ability of the rules produced by the algorithms implemented. Bias selection issues in machine-learning are now the subject of much attention in the literature, see  \cite{BCZSK16}, \cite{ZWYOC17}, \cite{BHSDR19}, \cite{Liu2016TransferLB} or
\cite{huang2007correcting}. In the context of face analysis, a research area
including a broad range of applications such as face detection, face
recognition or face attribute detection, machine learning algorithms trained
with baised training data, \textit{e.g.} in terms of gender or ethnicity, raise concerns
about fairness in machine learning. Unfair algorithms may induce systemic undesired
disadvantages for specific social groups,
see \cite{das2018mitigating} for further details. Several
examples of bias in deep learning based face recognition systems are discussed
in \cite{nagpal2019deep}.

Throughout the present article, we consider the case where the i.i.d. sample $Z'_1,\;\ldots,\; Z'_n$ available for training is not drawn from $P$ but from another distribution $P'$, with respect to which $P$ is absolutely continuous, and the goal pursued is to set theoretical grounds for the application of ideas behind Importance Sampling (IS in short) methodology to extend the ERM approach to this learning setup. We highlight that the problem under study is a very particular case of \textit{Transfer Learning} (see \textit{e.g.} \cite{5288526}, \cite{BenDavid10} and \cite{storkey2009training}), a research area currently receiving much attention in the literature and encompassing general situations where the information/knowledge one would like to transfer may take a form in the \textit{target} space very different from that in the \textit{source} space (referred to as \textit{domain adaptation}).

\noindent {\bf Weighted ERM (WERM).} In this paper, we investigate conditions guaranteeing that values for the parameter $\theta$ that nearly minimize $\eqref{eq:risk}$ can be obtained through minimization of a weighted version of the empirical risk based on the $Z'_i$'s, namely
\begin{equation}\label{eq:weight_emp_risk}
\widetilde{\cR}_{w,n}(\theta)=\cR_{\widetilde{P}_{w,n}}(\theta),
\end{equation}
where $\widetilde{P}_{w,n}=(1/n)\sum_{i=1}^nw_i\delta_{Z'_i}$ and $w=(w_1,\;\ldots,\; w_n)\in \mathbb{R}_+^n$ is a certain weight vector. Of course, ideal weights $w^*$ are given by the likelihood function $\Phi(z)=(dP/dP')(z)$: $w_i^*=\Phi(Z'_i)$ for $i\in \{1,\; \ldots,\; n \}$. In this case, the quantity \eqref{eq:weight_emp_risk} is obviously an unbiased estimate of the true risk \eqref{eq:risk}:
\begin{equation}
\mathbb{E}_{P'}\left[ \cR_{\widetilde{P}_{w^*,n}}(\theta)   \right]=\cR_P(\theta),
 \end{equation}
 and generalization bounds for the $\cR_{P}$-risk excess of minimizers of $\widetilde{\cR}_{w^*,n}$ can be directly established by studying the concentration properties of the empirical process related to the $Z'_i$'s and the class of functions $\{\Phi(\cdot)\ell(\theta,\; \cdot):\; \theta\in \Theta  \}$ (see section \ref{sec:background} below). However, the \textit{importance function} $\Phi$ is unknown in general, just like distribution $P$. It is the major purpose of this article to show that, in far from uncommon situations, the (ideal) weights $w_i^*$ can be estimated from the $Z_i'$s combined with auxiliary information on the target population $P$. As shall be seen below, such favorable cases include in particular classification problems where class probabilities in the test stage differ from those in the training step, risk minimization in stratified populations (see \cite{Bekker2018}), with strata statistically represented in a different manner in the test and training populations, positive-unlabeled learning (PU-learning, see \textit{e.g.} \cite{du2014analysis}). In each of these cases, we show that the stochastic process obtained by plugging the weight estimates in the weighted empirical risk functional \eqref{eq:weight_emp_risk} is much more complex than a simple empirical process (\textit{i.e.} a collection of i.i.d. averages) but can be however studied by means of \textit{linearization techniques}, in the spirit of the ERM extensions established in \cite{Clemencon08Ranking} or \cite{CV08NIPS1}. Learning rate bounds for minimizers of the corresponding risk estimate are proved and, beyond these theoretical guarantees, the performance of the weighted ERM approach is supported by convincing numerical results.

 The article is structured as follows. In section \ref{sec:background}, the ideal case where the importance function $\Phi$ is known is preliminarily considered and a first basic example where the optimal weights can be easily inferred and plugged into the risk without deteriorating the learning rate is discussed. The main results of the paper are stated in section \ref{sec:main}, which shows that the methodology promoted can be applied to two important problems in practice, risk minimization in stratified populations and PU-learning, with generalization guarantees. Illustrative numerical experiments are displayed in section \ref{sec:num}, while some concluding remarks are collected in section \ref{sec:conclusion}. Proofs and additional results are deferred to the Supplementary Material.

\section{Importance Sampling - Risk Minimization with Biased Data}
\label{sec:background}

Here and throughout, the indicator function of any event $\mathcal{E}$ is denoted by $\mathbb{I}\{\mathcal{E}\}$, the $\sup$ norm of any bounded function $h:\mathcal{Z}\rightarrow \mathbb{R}$ by $\vert\vert h\vert\vert_{\infty}$.
We place ourselves in the framework of statistical learning based on biased training data previously introduced. As a first go, we consider the unrealistic situation where the importance function $\Phi$ is known, insofar as we shall subsequently develop techniques aiming at mimicking the minimization of the ideally weighted empirical risk
\begin{equation}\label{eq:ideal_wer}
\widetilde{\cR}_{w^*,n}(\theta)=\frac{1}{n} \sum_{i=1}^n w_i^* \ell(\theta, Z_i'),
\end{equation}
namely the (unbiased) Importance Sampling estimator of \eqref{eq:risk} based on the instrumental data $Z'_1,\; \ldots,\; Z'_n$. The following result describes the performance of minimizers $\widetilde{\theta}^*_n$ of \eqref{eq:ideal_wer}. Since the goal of this paper is to promote the main ideas of the approach rather than to state results with the highest level of generality due to space limitations, we assume throughout the article for simplicity that $\ell$ and $\Phi$ are both bounded functions. For $\sigma_1,\; \ldots,\; \sigma_n$ independent Rademacher
random variables (\textit{i.e.} symmetric $\{ -1,1 \}$-valued r.v.'s), independent from the $Z'_i$'s, we define the Rademacher average associated to the class of function $\mathcal{F}$ as
$
R'_n(\mathcal{F}):=\mathbb{E}_{\mathbf{\sigma}}\left[\sup_{\theta\in \Theta}\frac{1}{n}\left\vert \sum_{i=1}^n\sigma_i \ell(\theta,Z'_i) \right\vert\right].
$
This quantity can be bounded by metric entropy methods under appropriate complexity assumptions on the class $\mathcal{F}$, it is for instance of order $O_{\mathbb{P}}(1/\sqrt{n})$ when $\mathcal{F}$ is a {\sc VC} major class with finite {\sc VC} dimension, see \textit{e.g.} \cite{Boucheron2005}.
\begin{lemma}\label{ERM1}
With probability at least $1-\delta$, we have: $\forall n\geq 1$,
\begin{equation*}
    \cR_P (\widetilde{\theta}^*_n) - \min_{\theta\in \Theta}\cR_P(\theta) \le
    4\vert\vert \Phi\vert\vert_{\infty} \E \left[R'_n(\mathcal{F})\right]
    + 2\vert\vert \Phi\vert\vert_{\infty}\sup_{(\theta,z)\in \Theta\times \mathcal{Z}} \ell(\theta, z) \sqrt{\frac{2 \log (1/\delta)}{n}}.
\end{equation*}
\end{lemma}
Of course, when $P'=P$, we have $\Phi\equiv 1$ and the bound stated above simply describes the performance of standard empirical risk minimizers. The proof is based on the standard bound
$$
 \cR_P (\widetilde{\theta}^*_n) - \min_{\theta\in \Theta}\cR_P(\theta)\leq 2\sup_{\theta\in\Theta}\left\vert \widetilde{\cR}_{w^*,n}(\theta)- \mathbb{E}\left[ \widetilde{\cR}_{w^*,n}(\theta) \right]  \right\vert,
$$
combined with basic concentration results for empirical processes, see the Supplementary Material for further details. Of course, the importance function $\Phi$ is generally unknown and must be estimated in practice. As illustrated by the elementary example below (related to binary classification, in the situation where the probability of occurence of a positive instance significantly differs in the training and test stages), in certain statistical learning problems with biased training distribution, $\Phi$ takes a simplistic form and can be easily estimated from the $Z'_i$'s combined with auxiliary information on $P$.

\noindent {\bf Binary classification with varying class probabilities.} The
flagship problem in supervised learning corresponds to the simplest situation,
where $Z=(X,Y)$, $Y$ being a binary variable valued in $\{-1,+1 \}$ say, and
the r.v. $X$ takes its values in a measurable space $\mathcal{X}$ and models
some information hopefully useful to predict $Y$. The parameter space $\Theta$
is a set $\mathcal{G}$ of measurable mappings (\textit{i.e.} classifiers)
$g:\mathcal{X}\to \{-1,\; +1\}$ and the loss function is given by $\ell(g,\;
(x,y))=\mathbb{I}\{ g(x) \neq y  \}$ for all $g$ in $\mathcal{G}$ and any
$(x,y)\in \mathcal{X}\times \{-1,\; +1 \}$. The distribution $P$ of the random
pair $(X,Y)$ can be either described by $X$'s marginal distribution $\mu(dx)$
and the posterior probability $\eta(x)=\mathbb{P}\{Y=+1\mid X=x  \}$ or else by
the triplet $(p, F_+,F_-)$ where $p=\mathbb{P}\{Y=+1  \}$ and $F_{\sigma}(dx)$
is $X$'s conditional distribution given $Y=\sigma 1$ with $\sigma\in\{-,\;
+\}$. It is very common that the fraction of positive instances in the training
dataset is significantly lower than the rate $p$ expected in the test stage,
supposed to be known here (see the Supplementary Material for the case where the rate $p$ is only approximately known).
We thus consider the case where the distribution
$P'$ of the training data $(X'_1,Y'_1),\; \ldots,\, (X'_n,Y'_n)$ is described
by the triplet $(p', F_+,F_-)$ with $p'<p$. The likelihood function takes the
simple following form $$
    \Phi(x,y)=
    \mathbb{I}\{ y=+1 \}\frac{p}{p'}
    + \mathbb{I}\{ y=-1 \}\frac{1-p}{1-p'}\overset{def}{=}\phi(y),
$$
which reveals that it depends on the label $y$ solely, and the ideally weighted empirical risk process is
\begin{equation}\label{eq:th_weight_risk}
    \widetilde{\cR}_{w^*,n}(g)
    =\frac{p}{p'}\frac{1}{n}\sum_{i: Y'_i=1}\mathbb{I}\{g(X'_i)=-1  \}
    + \frac{1-p}{1-p'}\frac{1}{n}\sum_{i: Y'_i=-1}\mathbb{I}\{g(X'_i)=+1  \}.
\end{equation}
In general the theoretical rate $p'$ is unknown and one replaces
\eqref{eq:th_weight_risk} with

\begin{equation}\label{eq:emp_weighted_risk}
    \widetilde{\cR}_{\widehat{w}^*,n}(g)
    =\frac{p}{n'_+}\sum_{i: Y'_i=1}\mathbb{I}\{g(X'_i)=-1  \}
    + \frac{1-p}{n'_-}\sum_{i: Y'_i=-1}\mathbb{I}\{g(X'_i)=+1  \},
\end{equation}
where $n'_+=\sum_{i=1}^n\mathbb{I}\{Y'_i=+1  \}=n-n'_-$,
$\widehat{w}_i^*=\widehat{\phi}(Y'_i)$ and
$\widehat{\phi}(y)=\mathbb{I}\{ y=+1 \}np/n' _++ \mathbb{I}\{ y=-1 \}n(1-p)/n'_-
$. The stochastic process above is not a standard empirical process but a collection of sums of two ratios of basic averages. However, the following result provides a uniform control of the deviations between the ideally weighted empirical risk and that obtained by plugging the empirical weights into the latter.
\begin{lemma}\label{lem:approx1}
Let $\varepsilon\in(0,\; 1/2)$. Suppose that $p'\in (\varepsilon,\; 1-\varepsilon)$. For any $\delta\in (0,1)$, we have with probability larger than $1-\delta$:
$$
\sup_{g\in \mathcal{G}}\left\vert  \widetilde{\cR}_{\widehat{w}^*,n}(g)-    \widetilde{\cR}_{w^*,n}(g) \right\vert
\leq \frac{2}{\varepsilon^2} \sqrt{\frac{\log (2/\delta)}{2n}},
$$
as soon as $n\geq 2\log(2/\delta)/\varepsilon^2$.
\end{lemma}
See the Appendix for the technical proof. Consequently, minimizing \eqref{eq:emp_weighted_risk} nearly boils down to minimizing \eqref{eq:th_weight_risk}. Combining Lemmas \ref{lem:approx1} and \ref{ERM1}, we immediately get the generalization bound stated in the result below.
\begin{corollary}\label{cor1}
Suppose that the hypotheses of Lemma \ref{lem:approx1} are fulfilled. Let $\widetilde{g}_n$ be any minimizer of
$\widetilde{\cR}_{\widehat{w}^*,n}$ over class $\mathcal{G}$. We
have with probability at least $1- \delta$:
\begin{equation*}
    \cR_P (\widetilde{g}_n) - \inf_{g\in \mathcal{G}}\cR_P(g)
    \leq \frac{2\max(p, 1-p) }{\varepsilon}
    \left( 2 \mathbb{E}[R'_n(\mathcal{G})  ]     +\sqrt{\frac{2\log(2/\delta)}{n}} \right)
+ \frac{4}{\varepsilon^2} \sqrt{\frac{\log (4/\delta)}{2n}},
\end{equation*}
as soon as $n\geq 2\log(4/\delta)/\varepsilon^2$; where $R'_n(\mathcal{G})=(1/n)\mathbb{E}_{\mathbf{\sigma}}[\sup_{g\in \mathcal{G}}\vert \sum_{i=1}^n\sigma_i \mathbb{I}\{g(X'_i)\neq Y'_i\} \vert]$.
\end{corollary}

Hence, some side information (\textit{i.e.} knowledge of parameter $p$) has permitted to weight the training data in order to build an empirical risk functional that approximates the target risk and to show that minimization of this risk estimate yields prediction rules with optimal (in the minimax sense) learning rates. The purpose of the subsequent analysis is to show that this remains true for more general problems. Observe in addition that the bound in Corollary \ref{cor1} deteriorates as $\varepsilon$ decays to zero: the method used here is not intended to solve the \textit{few shot} learning problem, where almost no training data with positive labels is available (\textit{i.e.} $p'\approx 0$). As shall be seen in subsection \ref{subsec:pu}, alternative estimators of the importance function must be considered in this situation.
\begin{remark}
Although the quantity \eqref{eq:emp_weighted_risk} can be viewed as a \textit{cost-sensitive} version of the empirical classification risk based on the $(X'_i,Y'_i)$'s (see \textit{e.g.} \cite{BHH06}), we point out that the goal pursued here is not  to achieve an appropriate trade-off between type I and type II errors in the $P'$ classification problem as in biometric applications for instance (\textit{i.e.} optimization of the $(F_+,F_-)$-{\sc ROC} curve at a specific point) but to transfer knowledge gained in analyzing the biased data drawn from $P'$ to the classification problem related to distribution $P$.
\end{remark}

\noindent {\bf Related work.} We point out that the natural idea of using weights in ERM problems that mimic those induced by the importance function has already been used in \cite{sugiyama2008direct} for \textit{covariate shift adaptation} problems (\textit{i.e.} supervised situations, where the conditional distribution of the output given the input information is the same in the training and test domains), when, in contrast to the framework considered here, a test sample is additionally available (a method for estimating directly the importance function based on Kullback-Leibler divergence minimization is proposed, avoiding estimation of the test density).
Importance sampling estimators have been also considered in \cite{garcke2014importance} in the setup of \textit{inductive transfer learning} (the tasks between source and target are different, regardless of the
similarities between source and target domains), where the authors have proposed two methods to approximate the importance function, among which one is again based on minimizing the Kullback-Leibler divergence between the two distributions.
In \cite{cortes2008sample}, the sample selection bias is assumed to be independent from the label, which is not true under our stratum-shift assumption or for the PU learning problem (see section \ref{sec:main}).
Lemma \ref{ERM1} assumes that the exact importance function is known, as does \cite{cortes2010learning}. The next section introduces new results for more realistic settings where it has to be learned from the data.

\section{Weighted Empirical Risk Minimization - Generalization Guarantees}
\label{sec:main}

Through two important and generic examples, relevant for many applications, we show that the approach sketched above can be applied to general situations, where appropriate auxiliary information on the target distribution is available, with generalization guarantees. 

\subsection{Statistical Learning from Biased Data in a Stratified Population}\label{sec:stratified_population}

A natural extension of the simplistic problem considered in section \ref{sec:background} is multiclass classification in a stratified population. The random labels $Y$ and $Y'$ are supposed to take their values in $\{1,\; \ldots,\; J\}$ say, with $J \geq 1$, and each labeled observation $(X,Y)$ belongs to a certain random stratum $S$ in $\{1,\; \ldots,\; K \}$ with $K\geq 1$. Again, the distribution $P$ of a random element $Z=(X,Y,S)$ may be described by the parameters $\{(p_{j,k},F_{j,k} ):\; 1 \leq j \leq J,\; 1 \leq k \leq K\}$ where $F_{j,k}$  is the conditional distribution of $X$ given $(Y,S)=(j,k)$ and $p_{j,k} = \mathbb{P}_{(X,Y,S)\sim P}\{Y=j, S=k\}$. Then, we have
$$
dP(x, y, s) =
\sum_{j=1}^J \sum_{k=1}^K \mathbb{I}\{y=j, s=k\} p_{j,k} dF_{j,k}(x),
$$
and considering a distribution $P'$ with $F_{j,k} \equiv F'_{j,k}$ but possibly different class-stratum probabilities $p'_{j,k}$, the likelihood function becomes
$$
    \frac{dP}{dP'}(x,y,s)=
    \sum_{j=1}^J \sum_{k=1}^K  \frac{p_{j,k}}{p'_{j,k}} \mathbb{I}\{ y=j,s=k\}
    \overset{def}{=} \phi(y,s).
$$
A more general framework can actually encompass this specific setup by
defining 'meta-strata' in $\{1,\; \ldots,\; J\}\times \{1,\; \ldots,\; K\}$.
Strata may often correspond to categorical input features in practice. The
formalism introduced below is more general and includes the example considered
in the preceding section, where strata are defined by labels.

\noindent {\bf Learning from biased stratified data.}
Consider a general mixture model, where distributions $P$ and $P'$ are stratified over
$K\ge 1$ strata.  Namely, $Z=(X, S)$ and $Z'=(X', S')$ with auxiliary random variables $S$
and $S'$ (the strata) valued in $\{1,\; \ldots,\; K\}$.
We place ourselves in a \textit{stratum-shift} context, assuming that the conditional distribution of $X$ given $S=k$ is the same as that
of $X'$ given $S'=k$, denoted by $F_k(dx)$, for any $k\in\{1,\; \ldots,\; K \}$.  However, stratum probabilities $p_k = \mathbb{P}(S=k)$ and $p'_k = \mathbb{P}(S'=k)$ may possibly be different.
In this setup, the likelihood function depends only on the strata and can be expressed in a very simple form, as follows:
$$
    \frac{dP}{dP'}(x,s)=
    \sum_{k=1}^K \mathbb{I}\{s=k\} \frac{p_k}{p'_k}\overset{def}{=} \phi(s).
$$
In this case, the ideally weighted empirical risk writes
$$
    \widetilde{\cR}_{w^*,n}(\theta)=
    \frac{1}{n} \sum_{i=1}^n
    \ell(\theta,Z'_i) \sum_{k=1}^K \mathbb{I}\{S_i'=k\} \frac{p_k}{p'_k}.
$$
If the strata probabilities $p_k$'s for the test distribution are known,
an empirical counterpart of the ideal empirical risk above is obtained by simply plugging estimates of the $p'_k$'s computed from
the training data:
\begin{equation} \label{eq:strata_reweighting}
    \widetilde{\cR}_{\widehat{w}^*,n}(\theta)=
    \sum_{i=1}^n  \ell(\theta, Z'_i)
    \sum_{k=1}^K \mathbb{I}\{S_i'=k\}
    \frac{p_k }{n'_k},
\end{equation}
with $n'_k = \sum_{i=1}^n \mathbb{I}\{S'_i=k\}$, $\widehat{w}_i^*=\widehat{\phi}(S'_i)$ and
$
\widehat{\phi}(s)= \sum_{k=1}^K \mathbb{I}\{s=k\}  n p_k/n'_k$.

A bound for the excess of risk
is given in Theorem \ref{th:excess_risk},
that can be viewed as a generalization of Corollary \ref{cor1}.

\begin{theorem}
  \label{th:excess_risk}
  Let $\varepsilon \in (0,1/2)$ and assume that $p'_k \in (\varepsilon, 1-\varepsilon)$ for $k=1, \ldots,K$. Let $\widetilde{\theta}^*_n$ be any minimizer of
  $\widetilde{\cR}_{\widehat{w}^*,n}$ as defined in \eqref{eq:strata_reweighting} over class $\Theta$. We
  have with probability at least $1- \delta$:
  \begin{equation*}
      \cR_P (\widetilde{\theta}^*_n) - \inf_{\theta\in \Theta}\cR_P(\theta)
      \leq \frac{2\max_k p_k }{\varepsilon}
      \left( 2  \mathbb{E}[R'_n(\mathcal{F})  ]  + L \sqrt{\frac{2\log(2/\delta)}{n}} \right)
+\frac{4 L}{\varepsilon^2} \sqrt{\frac{\log (4K/\delta)}{2n}},
\end{equation*}
  as soon as $n\geq 2\log(4K/\delta)/\epsilon^2$; where $R'_n(\mathcal{F})=(1/n)\mathbb{E}_{\mathbf{\sigma}}[\sup_{\theta\in \Theta}\vert \sum_{i=1}^n\sigma_i \ell(\theta, Z'_i) \vert]$, and the loss is bounded by $L=\sup_{(\theta, z) \in \Theta \times \mathcal{Z}} \ell(\theta, z)$.
\end{theorem}

Just like in Corollary \ref{cor1}, the bound in Theorem \ref{th:excess_risk} explodes when $\varepsilon$ vanishes, which corresponds to the situation where a stratum $k\in\{1,\dots,K\}$ is very poorly represented in the training data, \textit{i.e.} when $p'_k <<p_k$. Again, as highlighted by the experiments carried out, reweighting the losses in a frequentist (ERM) approach guarantees good generalization properties in a specific setup only, where the training information, though biased, is sufficiently informative.

\subsection{Positive-Unlabeled Learning}\label{subsec:pu}

Relaxing the \textit{stratum-shift} assumption made in the previous subsection, the importance function becomes more complex and writes:
$$
    \Phi(x, s) = \frac{dP}{dP'}(x,s)=
    \sum_{k=1}^K \mathbb{I}\{s=k\} \frac{p_k}{p'_k} \frac{dF_k}{dF_k'}(x),
$$
where $F_k$ and $F_k'$ are respectively the conditional distributions of $X$ given $S=k$ and of $X'$ given $S'=k$.
The Positive-Unlabeled (PU) learning problem, which has recently been the subject of much attention (see \textit{e.g.} \cite{du2014analysis}, \cite{DuPlessis2015}, \cite{kiryo2017positive}), provides a typical example of this situation. Re-using the notations introduced in section \ref{sec:background}, in the PU problem, the testing and training distributions $P$ and $P'$ are respectively described by the triplets $(p,F_+,F_-)$ and $(q,F_+,F)$,
where $F=pF_+ + (1-p)F_-$ is the marginal distribution of $X$. Hence, the objective pursued is to solve a binary classification task, based on the sole observation of a training sample pooling data with positive labels and unlabeled data, $q$ denoting the theoretical fraction of positive data among the dataset.
As noticed in \cite{du2014analysis} (see also \cite{DuPlessis2015}, \cite{kiryo2017positive}), the likelihood/importance function can be expressed in a simple manner, as follows:
\begin{equation}\label{eq:PU1}
\forall (x,y)\in\mathcal{X}\times \{-1,\; +1 \},\;\; \Phi(x,y) = \frac{p}{q}\mathbb{I}\{y=+1 \} + \frac{1}{1-q} \mathbb{I}\{y=-1 \} - \frac{p}{1-q} \frac{dF_+}{dF}(x) \mathbb{I}\{y=-1 \}.
\end{equation}
Based on an i.i.d. sample $(X'_1,Y'_1),\; \ldots,\; (X'_n,Y'_n)$ drawn from
$P'$ combined with the knowledge of $p$ (which can also be estimated from PU data, see e.g. \cite{du2014class}) and using that $F_-=(1/(1-p))(F-pF_+)$,
one may obtain estimators of $q$, $F_+$ and $F$ by computing
$n'_+/n=(1/n)\sum_{i=1}^n\mathbb{I}\{Y'_i=+1 \}$,
$\widehat{F}_+=(1/n'_+)\sum_{i=1}^n\mathbb{I}\{Y'_i=+1\}\delta_{X'_i}$ and
$\widehat{F} =(1/n'_-)\sum_{i=1}^n\mathbb{I}\{Y'_i=-1\}\delta_{X'_i}$. However,
plugging these quantities into \eqref{eq:PU1} do not permit to get a
statistical version of the importance function, insofar as the probability
measures $\widehat{F}_+$ and $\widehat{F}$ are mutually singular with
probability one, as soon as $F_+$ is continuous. Of course, as proposed in
\cite{du2014analysis}, one may use statistical methods (\textit{e.g.} kernel
smoothing) to build distribution estimators, that ensures absolute continuity
but are subject to the curse of dimensionality. However, WERM can still be
applied in this case, by observing that: $\forall g\in \mathcal{G}$,
\begin{equation}
\cR_P(g)=-p+\mathbb{E}_{P'}\left[ \frac{2p}{q}\mathbb{I}\{g(X')=-1,\; Y'=+1  \} + \frac{1}{1-q} \mathbb{I}\{g(X')=+1,\; Y'=-1  \} \right],
\end{equation}
which leads to the weighted empirical risk
\begin{equation}\label{eq:PUrisk}
\frac{2p}{n'_+}\sum_{i:Y_i'=+1}\mathbb{I}\{g(X_i')=-1  \} + \frac{1}{n'_-}\sum_{i:Y_i'=-1} \mathbb{I}\{g(X_i')=+1  \}.
\end{equation}
Minimization of \eqref{eq:PUrisk} yields rules $\widetilde{g}_n$ whose generalization ability regarding the binary problem related to $(p,F_+,F_-)$ can be guaranteed, as shown by the following result, the form of the weighted empirical risk in this case being quite similar to \eqref{eq:emp_weighted_risk}.
\begin{theorem}
\label{th:excess_risk_PU}
Let $\varepsilon\in(0,\; 1/2)$. Suppose that $q\in (\varepsilon,\; 1-\varepsilon)$. Let $\widetilde{g}_n$ be any minimizer of
the weighted empirical risk \eqref{eq:PUrisk} over class $\mathcal{G}$.
We have with probability at least $1- \delta$:
\begin{equation*}
    \cR_P (\widetilde{g}_n) - \inf_{g\in \mathcal{G}}\cR_P(g)
    \leq \frac{2\max(2 p, 1) }{\varepsilon}
    \left( 2 \mathbb{E}[R'_n(\mathcal{G})  ]     +\sqrt{\frac{2\log(2/\delta)}{n}} \right)
+ \frac{4(2p+1)}{\varepsilon^2} \sqrt{\frac{\log (4/\delta)}{2n}},
\end{equation*}
as soon as $n\geq 2\log(4/\delta)/\varepsilon^2$; where $R'_n(\mathcal{G})=(1/n)\mathbb{E}_{\mathbf{\sigma}}[\sup_{g\in \mathcal{G}}\vert \sum_{i=1}^n\sigma_i \mathbb{I}\{g(X'_i)\neq Y'_i\} \vert]$.
\end{theorem}

\begin{remark}
Let $\eta(x)=\mathbb{P}\{Y=+1\mid X=x  \}$ denote the posterior probability and recall that $(dF_+/dF_-)(x)=((1-p)/p)(\eta(x)/(1-\eta(x))$. Observing that
\begin{equation}
  \label{eq:PhiPU}
\Phi(x,y) = \frac{p}{q}\mathbb{I}\{y=+1 \}+ \frac{1-\eta(x)}{1-q}\mathbb{I}\{y=-1 \},
\end{equation}
in the case when an estimate $\widehat{\eta}(x)$ of $\eta(x)$ is available, one can perform WERM using the empirical weight function
\begin{equation}
  \label{eq:phi_pu_eta}
\widehat{\Phi}(x,y) = \frac{np}{n'_+}\mathbb{I}\{y=+1 \}+ \frac{1-\widehat{\eta}(x)}{1-n'_+/n}\mathbb{I}\{y=-1 \}.
\end{equation}
A bound that describes how this approach generalizes, depending on the accuracy
of estimate $\widehat{\eta}$, can be easily established, for more details refer
to Theorem 3 in the Supplementary Material, where it is also discussed how to
exploit such formulas in order to design incremental WERM procedures.
\end{remark}

\subsection{Learning from Censored Data}
\label{subsec:censor}

Another important example of sample bias is the censorship setting where the learner has only
access to (right) censored targets $\min(Y', C')$ instead of $Y'$. Intuitively, this situation occurs when
$Y'$ is a duration/date, e.g. the date of death of a patient modeled by covariates $X'$,
and the study happens at a (random) date $C'$. Hence if $C'\le Y'$, then we know that the patient is still alive
at time $C'$ but the target time $Y'$ remains unknown.
This problem has been extensively studied (see e.g. \cite{fleming2011counting}, \cite{andersen2012statistical} and the references therein for the asymptotic theory and \cite{ausset2019empirical} for finite-time guarantees): we show here that it is an instance of WERM.
Formally, we respectively denote by $P$ and $P'$ the testing and training distributions of the r.v.'s $(X, \min(Y, C), \mathbb{I}\{Y\le C\})$ and $(X', \min(Y', C'), \mathbb{I}\{Y'\le C'\})$
both valued in $\mathbb{R}^d\times\mathbb{R}_+\times\{0, 1\}$ (with $Y, Y', C, C'$ all nonnegative r.v.'s) and such that the pairs $(X,Y)$ and $(X',Y')$ share the same distribution $Q$.
Moreover, $C>Y$ with probability $1$ (i.e. the testing data are never censored)
and $Y'$ and $C'$ are assumed to be conditionally independent given $X'$.
Hence, for all $(x,y,\delta)\in\mathbb{R}^d\times\mathbb{R}_+\times\{0, 1\}$:
$$
dP(x, y, \delta) = \delta dQ(x,y)
$$
and
$$
\delta dP'(x, y, \delta) = \delta \mathbb{P}(C'\ge y) d\mathbb{P}(X'=x, Y'=y | C'\ge y) = \delta S_{C'}(y|x) dQ(x,y),
$$
where $S_{C'}(y|x) = \mathbb{P}(C'\ge y | X'=x)$ denotes the conditional survival function of $C'$ given $X'$.
Then, the importance function is:
$$
\forall (x, y, \delta)\in\mathbb{R}^d\times\mathbb{R}_+\times\{0, 1\}, \quad \Phi(x, y, \delta) = \frac{dP}{dP'}(x, y, \delta) = \frac{\delta}{S_{C'}(y|x)}.
$$
In survival analysis, the ratio $\delta/S_{C'}(y|x)$ is called IPCW (inverse of the probability of censoring weight)
and $S_{C'}(y|x)$ can be estimated by using the Kaplan-Meier approach, see \cite{kaplan1958nonparametric}.
 
\section{Numerical Experiments}
\label{sec:num}

This section illustrates the impact of reweighting by the likelihood ratio on
classification performances, as a special case of the general strategy presented
in \Cref{sec:background}. A first simple illustration on known probability distributions
highlights the impact of the shapes of the distributions on the importance of
reweighting. This example illustrates in the infinite-sample case that
separable or almost separable data do not require reweighting, in contrast
to noisy data.  Since the distribution shapes are unknown for real data, we
infer that reweighting will have variable effectiveness, depending on the
dataset. This illustration is deferred to the Appendix.
We detail here a second experiment that uses the structure of
ImageNet to illustrate reweighting with a stratified population and strata
distribution bias or \emph{strata bias}.

We focus on the \textit{learning
from biased stratified data} setting introduced in \Cref{sec:stratified_population}
by leveraging the ImageNet Large Scale Visual Recognition Challenge
(ILSVRC); a well-known benchmark for the image classification task, see
\cite{DBLP:journals/corr/RussakovskyDSKSMHKKBBF14} for more details.

The challenge consists in
learning a classifier from 1.3 million training images spread out over 1,000
classes. Performance is evaluated using the validation dataset of 50,000
images of ILSVRC as our test dataset. ImageNet is an image database
organized according to the WordNet hierarchy, which groups nouns in sets
of related words called synsets. In that context, images are examples of
very precise nouns, e.g. \emph{flamingo}, which are contained in a larger
synset, e.g. \emph{bird}.

The impact of reweighting in presence of strata bias is illustrated on the ILSVRC
classification problem with broad significance synsets for strata. To do this,
we encode the data using deep neural networks. Specifically our encoding is the
flattened output of the last convolutional layer of the network ResNet50
introduced in \cite{DBLP:journals/corr/HeZRS15}. It was trained for
classification on the training dataset of ILSVRC. The encodings
$X_1, \dots, X_n$ belong to a 2,048-dimensional space.

A total of 33 strata are derived from a list of high-level categories provided
by ImageNet\footnote{\url{http://www.image-net.org/about-stats}}. The
construction of the strata is postponed to the Appendix. By default, strata
probabilities $p_k$ and $p_k'$ for $1 \le k \le K$ are equivalent between
training and testing datasets, meaning that reweighting by $\Phi$ would have
little to no effect. Since our testing data is the validation data of ILSVRC,
we have around $25$ times more training than testing data. Introducing a strata
bias parameter $0 \le \gamma \le 1$, we set the strata train probabilities such
that $p_k' = \gamma^{1 - \lfloor K/2 \rfloor/k} p_k$ before renormalization and
remove train instances so that the train set has the right distribution over
strata; see the Appendix for more details on the generation of strata bias.
When $\gamma$ is close to one, there is little to no strata bias. In contrast,
when $\gamma$ approaches $0$, strata bias is extreme.

\begin{figure}[!htb]
\minipage{0.35\textwidth}

  \includegraphics[width=\linewidth]{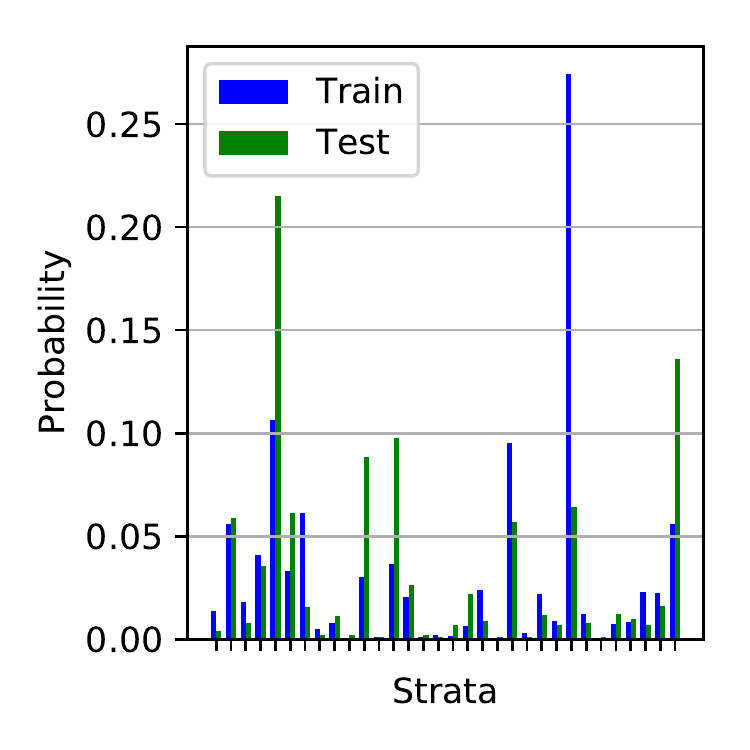}
\caption*{Comparison of $p_k$'s and $p_k'$'s.}\label{fig:prob_diff_ImageNet}
\endminipage\hfill
\minipage{0.6\textwidth}\center
    \begin{tabular}[t]{lccc}
	\toprule
	 Model & Reweighting & miss rate & top-5 error \\
	 \midrule
	 \multirow{4}{*}{Linear} & Unif. $\widehat{\Phi}=1$ & $0.344$ & $0.130$ \\
	    \cmidrule(lr){2-4}
	    & Strata  $\widehat{\Phi}$ & $ \mathbf{0.329}  $ & $\mathbf{0.120}  $ \\
	    \cmidrule(lr){2-4}
	    & \textcolor{gray}{Class $\widehat{\Phi}$} & \textcolor{gray}{ $0.328$ }
	    & \textcolor{gray}{ $0.119$ } \\
	    \cmidrule(lr){2-4}
	    & \textcolor{gray}{No bias} & \textcolor{gray}{ $0.297$ } &
	    \textcolor{gray}{ $0.102$ } \\
	    \midrule
	    \multirow{4}{*}{MLP} & Unif. $\widehat{\Phi} = 1$ & $0.371$ & $0.143$ \\
	    \cmidrule(lr){2-4}
	    & Strata $\widehat{\Phi}$ & $\mathbf{0.364}  $ & $\mathbf{0.138}  $ \\
	    \cmidrule(lr){2-4}
	    & \textcolor{gray}{Class $\widehat{\Phi}$} & \textcolor{gray}{ $0.363$ } &
	    \textcolor{gray}{ $0.138  $ } \\
	    \cmidrule(lr){2-4}
	    & \textcolor{gray}{No bias} & \textcolor{gray}{ $0.316  $ } &
	    \textcolor{gray}{ $0.111  $ } \\
\bottomrule
    \end{tabular}
    \caption*{Table of results.}\label{fig:prob_diff_ImageNet}
\endminipage
\caption{Results for the strata reweighting experiment with ImageNet.}\label{fig:ImageNet_res}
\end{figure}

\begin{figure}[!htb]
\minipage{0.33\textwidth}

    \includegraphics[width=\linewidth]{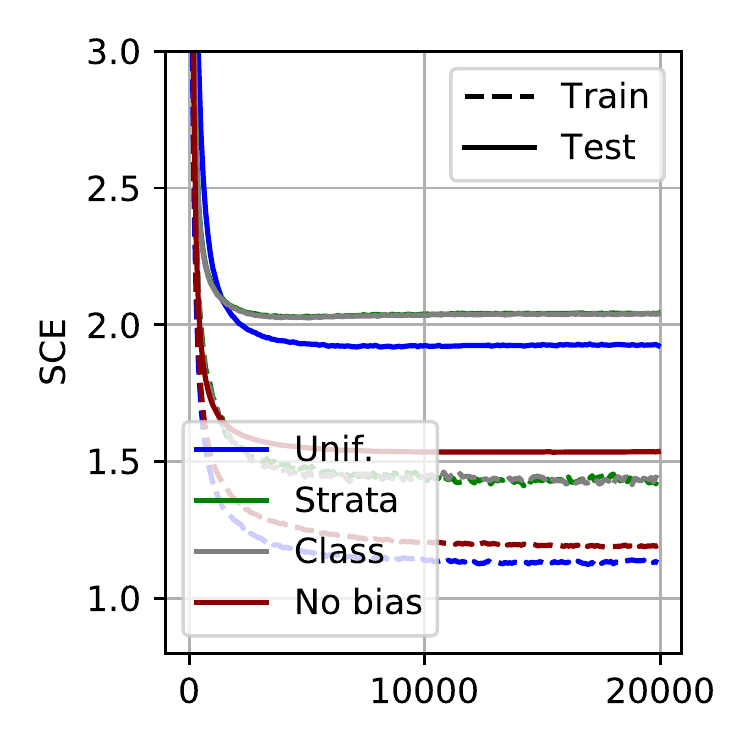}
\caption*{Dynamics for the SCE.}\label{fig:imgnetlin_quant_cost}
\endminipage
\hfill
\minipage{0.33\textwidth}

    \includegraphics[width=\linewidth]{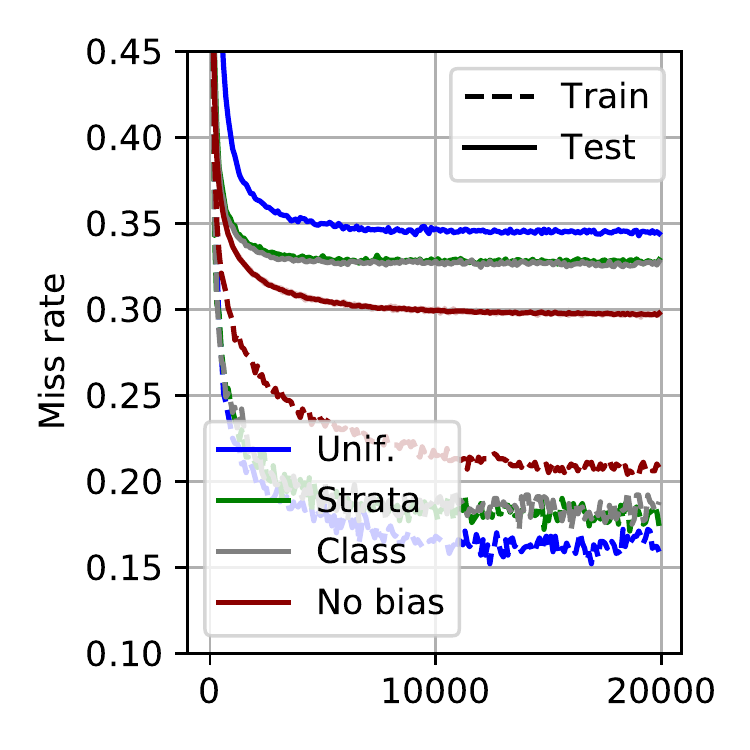}
\caption*{Dynamics for the miss rate.}\label{fig:imgnetlin_quant_acc}
\endminipage
\hfill
\minipage{0.33\textwidth}

    \includegraphics[width=\linewidth]{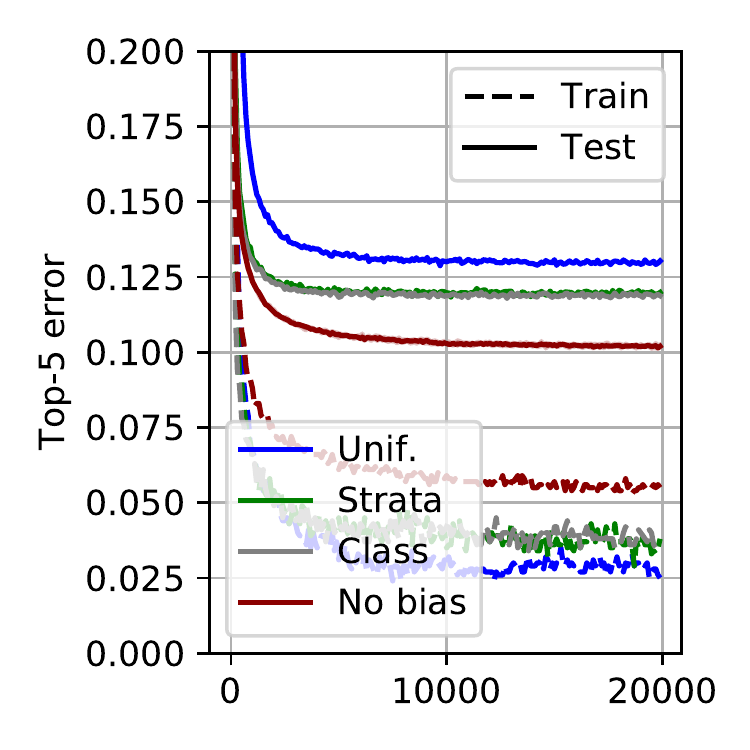}
\caption*{Dynamics for the top-5 error.}\label{fig:imgnetlin_quant_topk}
\endminipage

\caption{Dynamics for the linear model for the strata reweighting experiment with ImageNet.}\label{fig:ImageNet_lin_dynamics}
\end{figure}

\begin{figure}[!htb]
\minipage{0.33\textwidth}

    \includegraphics[width=\linewidth]{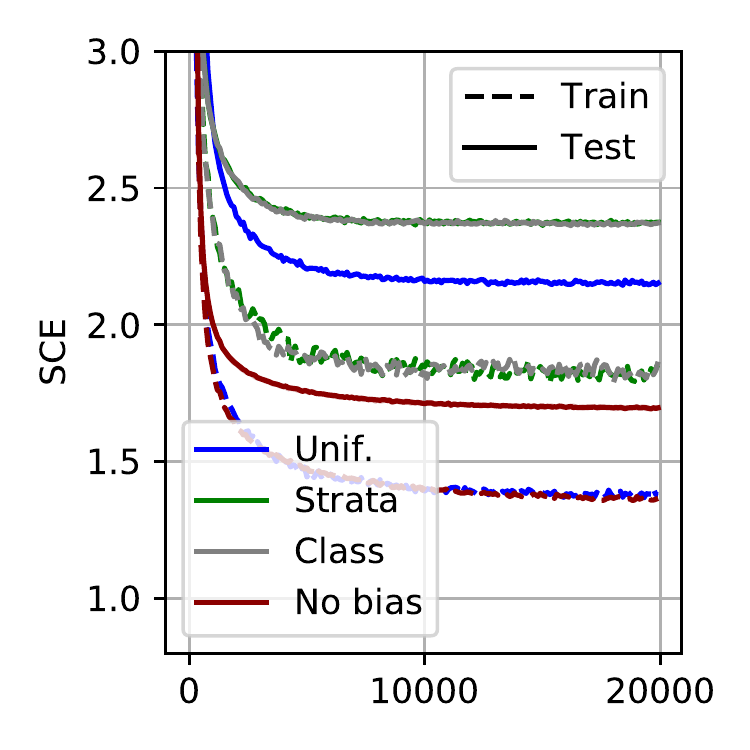}
\caption*{Dynamics for the SCE.}\label{fig:imgnetmlp_quant_cost}
\endminipage
\hfill
\minipage{0.33\textwidth}

    \includegraphics[width=\linewidth]{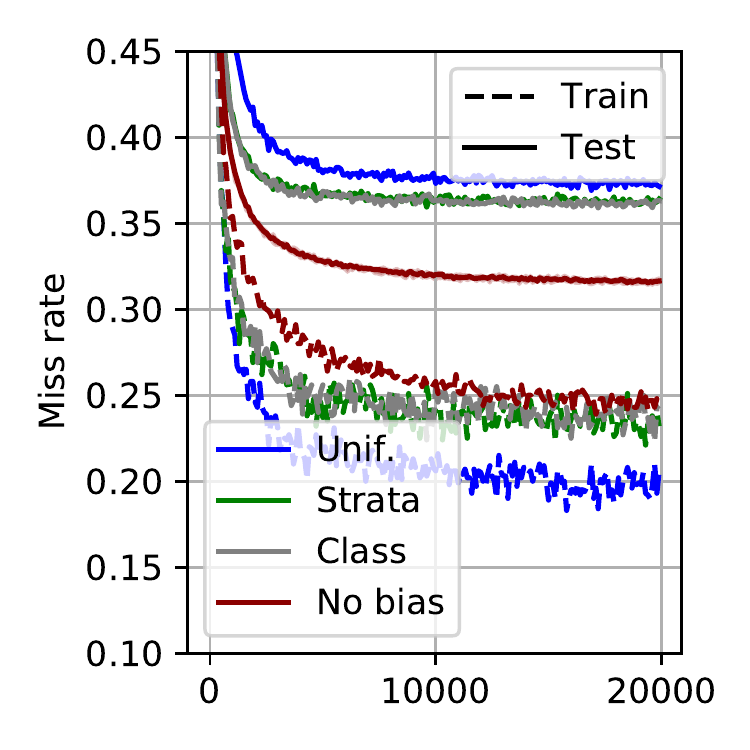}
\caption*{Dynamics for the miss rate.}\label{fig:imgnetmlp_quant_acc}
\endminipage
\hfill
\minipage{0.33\textwidth}

    \includegraphics[width=\linewidth]{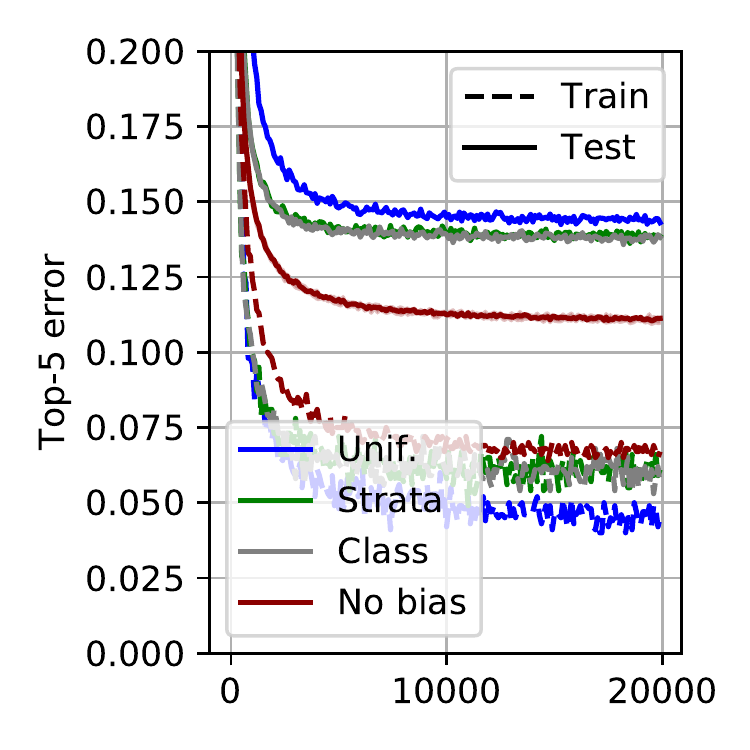}
\caption*{Dynamics for the top-5 error.}\label{fig:imgnetmlp_quant_topk}
\endminipage

\caption{Dynamics for the MLP model for the strata reweighting experiment with ImageNet.}\label{fig:ImageNet_mlp_dynamics}
\end{figure}

The models used are a linear model and a multilayer perceptron (MLP) with one
hidden layer; more details are given in the Appendix. We report
better performance when reweighting using the strata information, compared to
the case where the strata information is ignored, see \cref{fig:ImageNet_res}.
For comparison, we added two reference experiments: one which reweights the
train instances by the class probabilities, which we do not know in a
stratified population experiment, and one with more data and no strata bias
because it uses all of the ILSVRC train data. The dominance of the linear
model over the MLP can be justified by the much higher number of parameters
to estimate  
\section{Conclusion}\label{sec:conclusion}
In this paper, we have considered specific transfer learning problems, where
the distribution of the test data $P$ differs from that of the training data,
$P'$, and is absolutely continuous with respect to the latter. This setup
encompasses many situations in practice, where the data acquisition process is
not perfectly controlled. In this situation, a simple change of measure shows
that the target risk may be viewed as the expectation of a weighted version of
the basic empirical risk, with ideal weights given by the importance function
$\Phi=dP/dP'$, unknown in practice. Throughout this article, we have shown that,
in statistical learning problems corresponding to a wide variety of practical
applications, these ideal weights can be replaced by statistical versions based
solely on the training data combined with very simple information about the
target distribution. 
The generalisation capacity of rules learnt from biased training data by
minimization of the weighted empirical risk has been established, 
with learning bounds.
These theoretical results are also illustrated with
several numerical experiments.
 
\newpage

\bibliographystyle{apalike}

\newpage

\section*{Appendix - Technical Proofs}

Here we detail the proofs of the results stated in the article and discuss their connection with related work.

\subsection*{Proof of Lemma \ref{ERM1}}
Let $\delta\in (0,1)$. Applying the classic maximal deviation bound stated in Theorem 3.2 of \cite{Boucheron2005} to the bounded class $\mathcal{K}=\{z\in\mathcal{Z}\mapsto \Phi(z)l(\theta, z):\;\; \theta\in \Theta  \}$, we obtain that, with probability at least $1-\delta$:
$$
\sup_{\theta\in\Theta}\left\vert \widetilde{\cR}_{w^*,n}(\theta)- \mathbb{E}\left[ \widetilde{\cR}_{w^*,n}(\theta) \right]  \right\vert \leq 2\mathbb{E}\left[R'_n(\mathcal{K})  \right]+\vert\vert \Phi\vert\vert_{\infty}\sup_{(\theta,z)\in \Theta\times \mathcal{Z}}\left\vert \ell(\theta, z) \right\vert \sqrt{\frac{2\log (1/\delta)}{n}}.
$$
In addition, by virtue of the contraction principle, we have $R'_n(\mathcal{K})\leq \vert\vert \Phi\vert\vert_{\infty} R'_n(\mathcal{F})$ almost-surely. The desired result can be thus deduced from the bound above combined with the classic bound
$$
 \cR_P (\tilde{\theta}^*_n) - \min_{\theta\in \Theta}\cR_P(\theta)\leq 2\sup_{\theta\in\Theta}\left\vert \widetilde{\cR}_{w^*,n}(\theta)- \mathbb{E}\left[ \widetilde{\cR}_{w^*,n}(\theta) \right]  \right\vert.
$$

\subsection*{Proof of Lemma \ref{lem:approx1}}
Apply twice the Taylor expansion
$$
\frac{1}{x}=\frac{1}{a}-\frac{x-a}{a^2}+\frac{(x-a)^2}{xa^2},
$$
so as to get
\begin{eqnarray*}
\frac{1}{n_+'/n}&=&\frac{1}{p'}-\frac{n_+'/n-p'}{p'^2}+\frac{(n_+'/n-p')^2}{p'^2 n'_+/n},\\
\frac{1}{n_-'/n}&=&\frac{1}{1-p'}-\frac{n_-'/n-1+p'}{(1-p')^2}+\frac{(n_-'/n-1+p')^2}{(1-p')^2n'_-/n}.
\end{eqnarray*}
This yields the decomposition
\begin{multline*}
\widetilde{\cR}_{\widehat{w}^*,n}(g)-\widetilde{\cR}_{w^*,n}(g)=-\frac{p}{p'^2}\left(\frac{n'_+}{n}-p'\right)\frac{1}{n}\sum_{i=1}^n\mathbb{I}\{g(X'_i)=-1,\; Y'_i=+1 \}\\
-\frac{1-p}{(1-p')^2}\left(\frac{n'_-}{n}-1+p'\right)\frac{1}{n}\sum_{i=1}^n\mathbb{I}\{g(X'_i)=+1,\; Y'_i=-1 \}+\frac{p(n_+'/n-p')^2}{p'^2n'_+/n}\frac{1}{n}\sum_{i=1}^n\mathbb{I}\{g(X'_i)=-1,\; Y'_i=+1 \}\\
+\frac{(1-p)(n_-'/n-1+p')^2}{(1-p')^2n'_-/n}\frac{1}{n}\sum_{i=1}^n\mathbb{I}\{g(X'_i)=+1,\; Y'_i=-1 \}.
\end{multline*}
We deduce that
\begin{equation*}
\left\vert
\widetilde{\cR}_{\widehat{w}^*,n}(g)-\widetilde{\cR}_{w^*,n}(g)
\right\vert \leq \frac{\vert n'_+/n-p'\vert}{\varepsilon^2}\left( 1+\vert n'_+/n-p'\vert \left( \frac{p}{n'_+/n}+\frac{1-p}{1-n'_+/n} \right) \right).
\end{equation*}
By virtue of Hoeffding inequality, we obtain that, for any $\delta\in (0,1)$, we have with probability larger than $1-\delta$:
$$
\left\vert n'_+/n-p'\right\vert \leq \sqrt{\frac{\log(2/\delta)}{2n}},
$$
so that, in particular, $\min\{ n'_+/n,\; 1- n'_+/n\}\geq \varepsilon-\sqrt{\log(2/\delta)/(2n)}$. This yields the desired result.

\subsection*{Proof of Corollary \ref{cor1} }
Observe first that $\vert\vert \Phi \vert\vert_{\infty}\leq \max(p,\; 1-p)/\varepsilon$ and
$$
\cR_P(\widetilde{g}_n)-\inf_{g\in\mathcal{G}}\cR_P(g)\leq 2\sup_{g\in \mathcal{G}}\left\vert  \widetilde{\cR}_{\widehat{w}^*,n}(g)-    \widetilde{\cR}_{w^*,n}(g) \right\vert + 2\sup_{g\in \mathcal{G}}\left\vert  \widetilde{\cR}_{w^*,n}(g)-    \cR_{P}(g) \right\vert.
$$
The result then directly follows from the application of Lemmas \ref{ERM1}-\ref{lem:approx1} combined with the union bound.

\subsection*{Proof of Theorem \ref{th:excess_risk} }

Observe first that $\vert\vert \Phi \vert\vert_{\infty}\leq \max_k p_k/\varepsilon$ and
$$
\cR_P(\widetilde{\theta}^*_n)-\inf_{\theta\in\Theta}\cR_P(\theta)\leq 2\sup_{\theta\in \Theta}\left\vert  \widetilde{\cR}_{\widehat{w}^*,n}(\theta)-    \widetilde{\cR}_{w^*,n}(\theta) \right\vert + 2\sup_{\theta\in \Theta}\left\vert  \widetilde{\cR}_{w^*,n}(\theta)-    \cR_{P}(\theta) \right\vert.
$$
The result then directly follows from the application of Lemmas \ref{ERM1}-\ref{lem:approx2} combined with the union bound.

\begin{lemma}
  \label{lem:approx2}
  Let $\varepsilon\in(0,\; 1/2)$. Suppose that $p'_k\in (\varepsilon,\; 1-\varepsilon)$ for $k\in\{1,\; \ldots,\; K\}$. For any $\delta\in (0,1)$, we have with probability larger than $1-\delta$:
  $$
  \sup_{\theta \in \Theta}\left\vert  \widetilde{\cR}_{\widehat{w}^*,n}(\theta)-    \widetilde{\cR}_{w^*,n}(\theta) \right\vert \leq \frac{2L}{\varepsilon^2} \sqrt{\frac{\log (2K/\delta)}{2n}},
  $$
  as soon as $n\geq 2\log(2K/\delta)/\varepsilon^2$, where $L=\sup_{(\theta, z) \in \Theta \times \mathcal{Z}} \ell(\theta, z)$.
\end{lemma}

\begin{proof}

Apply the Taylor expansion
$$
\frac{1}{x}=\frac{1}{a}-\frac{x-a}{a^2}+\frac{(x-a)^2}{xa^2},
$$
so as to get for all $k\in\{1, \dots, K\}$
\begin{equation*}
\frac{1}{n'_k/n}=\frac{1}{p'_k}-\frac{n'_k/n-p'_k}{p'^2_k}+\frac{(n'_k/n-p'_k)^2}{p'^2_k n'_k/n}.
\end{equation*}
This yields the decomposition
\begin{equation*}
\widetilde{\cR}_{\widehat{w}^*,n}(\theta)-\widetilde{\cR}_{w^*,n}(\theta)
= \frac{1}{n} \sum_{i=1}^n \ell(\theta, Z'_i) \sum_{k=1}^K \mathbb{I}\{S'_i=k \} \left( -\frac{p_k}{p'^2_k}\left(\frac{n'_k}{n}-p'_k\right) + \frac{p_k(n'_k/n-p'_k)^2}{p'^2_k n'_k/n} \right).
\end{equation*}
We deduce that
\begin{equation*}
\left\vert
\widetilde{\cR}_{\widehat{w}^*,n}(\theta)-\widetilde{\cR}_{w^*,n}(\theta)
\right\vert \leq \frac{L}{\varepsilon^2} \sum_{k=1}^K \vert n'_k/n-p'_k\vert p_k \left( 1 + \frac{\vert n'_k/n-p'_k\vert}{n'_k/n} \right).
\end{equation*}
By virtue of Hoeffding inequality, we obtain that, for any $k\in\{1,\dots, K\}$ and $\delta\in (0,1)$, we have with probability larger than $1-\delta$:
$$
\left\vert n'_k/n-p'_k\right\vert \leq \sqrt{\frac{\log(2/\delta)}{2n}},
$$
so that, by a union bound, $\max_k \{ n'_k/n \}\geq \varepsilon-\sqrt{\log(2K/\delta)/(2n)}$. This yields the desired result.

\end{proof}

\subsection*{Proof of Theorem \ref{th:excess_risk_PU} }
Observe first that $\vert\vert \Phi \vert\vert_{\infty}\leq \max(2p,\; 1)/\varepsilon$ and
$$
\cR_P(\widetilde{g}_n)-\inf_{g\in\mathcal{G}}\cR_P(g)\leq 2\sup_{g\in \mathcal{G}}\left\vert  \widetilde{\cR}_{\widehat{w}^*,n}(g)-    \widetilde{\cR}_{w^*,n}(g) \right\vert + 2\sup_{g\in \mathcal{G}}\left\vert  \widetilde{\cR}_{w^*,n}(g)-    \cR_{P}(g) \right\vert,
$$
with weighted empirical risk $\widetilde{\cR}_{w^*,n}(g)$ defined in \eqref{eq:PUrisk}.
The result then directly follows from the application of Lemmas \ref{ERM1}-\ref{lem:approx3} combined with the union bound.

\begin{lemma}\label{lem:approx3}
Let $\varepsilon\in(0,\; 1/2)$. Suppose that $q\in (\varepsilon,\; 1-\varepsilon)$. For any $\delta\in (0,1)$, we have with probability larger than $1-\delta$:
$$
\sup_{g\in \mathcal{G}}\left\vert  \widetilde{\cR}_{\widehat{w}^*,n}(g)-    \widetilde{\cR}_{w^*,n}(g) \right\vert
\leq \frac{2(2p+1)}{\varepsilon^2} \sqrt{\frac{\log (2/\delta)}{2n}},
$$
as soon as $n\geq 2\log(2/\delta)/\varepsilon^2$.
\end{lemma}

\begin{proof}
Apply twice the Taylor expansion
$$
\frac{1}{x}=\frac{1}{a}-\frac{x-a}{a^2}+\frac{(x-a)^2}{xa^2},
$$
so as to get
\begin{eqnarray*}
\frac{1}{n_+'/n}&=&\frac{1}{q}-\frac{n_+'/n-q}{q^2}+\frac{(n_+'/n-q)^2}{q^2 n'_+/n},\\
\frac{1}{n_-'/n}&=&\frac{1}{1-q}-\frac{n_-'/n-1+q}{(1-q)^2}+\frac{(n_-'/n-1+q)^2}{(1-q)^2n'_-/n}.
\end{eqnarray*}
This yields the decomposition
\begin{multline*}
\widetilde{\cR}_{\widehat{w}^*,n}(g)-\widetilde{\cR}_{w^*,n}(g)=-\frac{2p}{q^2}\left(\frac{n'_+}{n}-q\right)\frac{1}{n}\sum_{i=1}^n\mathbb{I}\{g(X'_i)=-1,\; Y'_i=+1 \}\\
-\frac{1}{(1-q)^2}\left(\frac{n'_-}{n}-1+q\right)\frac{1}{n}\sum_{i=1}^n\mathbb{I}\{g(X'_i)=+1,\; Y'_i=-1 \}+\frac{2p(n_+'/n-q)^2}{q^2 n'_+/n}\frac{1}{n}\sum_{i=1}^n\mathbb{I}\{g(X'_i)=-1,\; Y'_i=+1 \}\\
+\frac{(n_-'/n-1+q)^2}{(1-q)^2n'_-/n}\frac{1}{n}\sum_{i=1}^n\mathbb{I}\{g(X'_i)=+1,\; Y'_i=-1 \}.
\end{multline*}
We deduce that
\begin{equation*}
\left\vert
\widetilde{\cR}_{\widehat{w}^*,n}(g)-\widetilde{\cR}_{w^*,n}(g)
\right\vert \leq \frac{\vert n'_+/n-q\vert}{\varepsilon^2}\left( 2p + 1 +\vert n'_+/n-q\vert \left( \frac{2p}{n'_+/n}+\frac{1}{1-n'_+/n} \right) \right).
\end{equation*}
By virtue of Hoeffding inequality, we obtain that, for any $\delta\in (0,1)$, we have with probability larger than $1-\delta$:
$$
\left\vert n'_+/n-q\right\vert \leq \sqrt{\frac{\log(2/\delta)}{2n}},
$$
so that, in particular, $\min\{ n'_+/n,\; 1- n'_+/n\}\geq \varepsilon-\sqrt{\log(2/\delta)/(2n)}$. This yields the desired result.
\end{proof}

\section*{Appendix - Inaccurate Prior Information about the Test Distribution}
\label{sec:inacc_proba}
As noticed in section \ref{sec:background}, it may happen that the rate of positive instances in the target population is approximately known only. Suppose that our guess for $p$ is $\widetilde{p}$ such that $\vert p-\widetilde{p} \vert\leq \zeta$, with $\zeta\in (0,1)$. Denote by $\widetilde{P}$ the distribution over $\mathcal{X}\times\{-1,+1\}$ under which $X$ is drawn from $\widetilde{p}F_++(1-\widetilde{p})F_-$ and such that $\mathbb{P}_{(X,Y)\sim\widetilde{P}}\{Y=1\mid X=x\}=\mathbb{P}_{(X,Y)\sim P}\{Y=1\mid X=x\}=\eta(x)$.

By a change of measure we have,
$$
\mathbb{P}_{\widetilde{P}}(Y\neq g(X)) = \mathbb{P}_{P}(Y\neq g(X))
+ \mathbb{E}_P\left[\left(\frac{d\widetilde{P}}{dP}(X,Y) - 1\right)\mathbb{I}\{Y\neq g(X)\}\right],
$$
which allows to bound the difference of the classification risks of $g$ under $P$ and $\widetilde{P}$:
$$
\left|\mathcal{R}_{\widetilde{P}}(g) - \mathcal{R}_P(g)\right| \le \mathbb{E}_P\left[\left|\frac{d\widetilde{P}}{dP}(X,Y) - 1\right|\right]
= 2 |\widetilde{p}-p| \le 2\zeta.
$$
 
\section*{Appendix - Additional Numerical Experiments}

In this Appendix, more details about the experiments carried out are provided.

\subsection*{Importance of reweighting for simple distributions}

Introduce a random pair $(X,Y)$ in $[0,1] \times \{-1, +1\}$ where $X\mid Y=+1$
has for probability density function (pdf) $f_+(x) = ( 1 + \alpha ) x^\alpha, \alpha>0$ and
$X\mid Y=-1$ has for pdf $f_-(x) = ( 1 + \beta
)(1-x)^\beta, \beta>0$. As in \Cref{sec:background}, the train and test datasets have
different class probabilities $p'$ and $p$ for $Y=+1$.
The loss $\ell$ is defined as $\ell(\theta,z) = \I\{ (x-\theta)y \ge 0\}$ where
$\theta>0$ is a learnt parameter.

The true risk can be explicitely calculated. For $\theta>0$, we have
\begin{align*}
    R_P(\theta) = p \theta^{1 + \alpha} + (1-p) (1-\theta)^{1+\beta},
\end{align*}
and the optimal threshold $\theta^*_{p}$ can be found by derivating the risk
$R_P(\theta)$.  The derivative is zero when $\theta$ satisfies
\begin{align}
    p(1+\alpha) \theta^\alpha = (1-p)(1+\beta)(1-\theta)^\beta.
    \label{optimal_threshold_sim_exp}
\end{align}
Solving \cref{optimal_threshold_sim_exp} is straightforward for well-chosen values of 
$\alpha, \beta$, which are detailed in \cref{fig:tab_sim_exp}. The excess error 
$\mathcal{E}(p',p) = R_P(\theta^*_{p'}) - R_P(\theta^*_{p})$
for the diagonal entries of \cref{fig:tab_sim_exp} are plotted in
\cref{fig:plot_sim_exp}, in the infinite sample case.

\begin{figure}[h]
    \centering
    \begin{tabular}[h]{lcccc}
	\toprule
	   	& \multicolumn{4}{c}{$(\alpha, \beta)$} \\
	 \cmidrule(lr){2-5}
	 & $(0,0)$ & $(1/2, 1/2)$ & $(1,1)$ & $(2,2)$ \\
	 \midrule
	 \smallskip
	$\theta^*_p$ & $[0,1]$ & $\frac{(1-p)^2}{p^2 + (1-p)^2}$ &
	$1-p$ &  $\frac{\sqrt{1-p}}{\sqrt{p} + \sqrt{1-p}}$\\
	 \bottomrule
    \end{tabular}
    \caption{Optimal parameters $\theta^*$ for different values of $\alpha, \beta$.}
    \label{fig:tab_sim_exp}
\end{figure}

\begin{figure}[h]
    \centering

    \minipage{0.25\textwidth}
	\includegraphics[width=\linewidth]{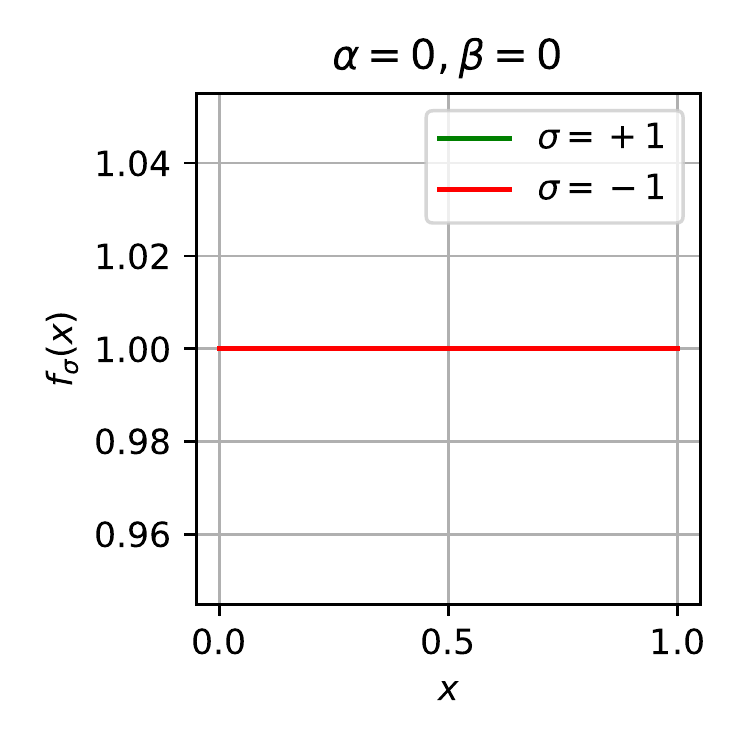}
    \endminipage
    \hfill
    \minipage{0.25\textwidth}
	\includegraphics[width=\linewidth]{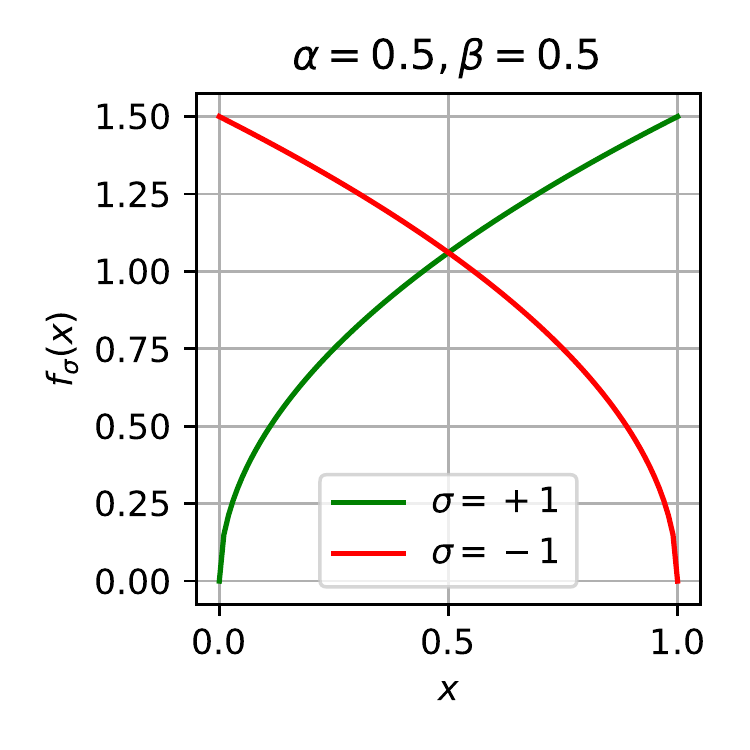}
    \endminipage
    \hfill
    \minipage{0.25\textwidth}
	\includegraphics[width=\linewidth]{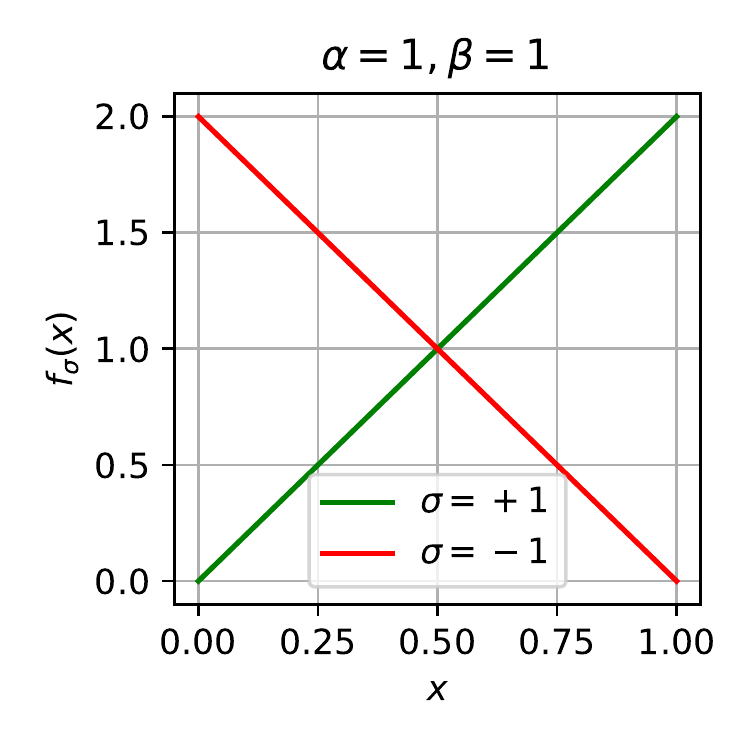}
    \endminipage
    \hfill
    \minipage{0.25\textwidth}
	\includegraphics[width=\linewidth]{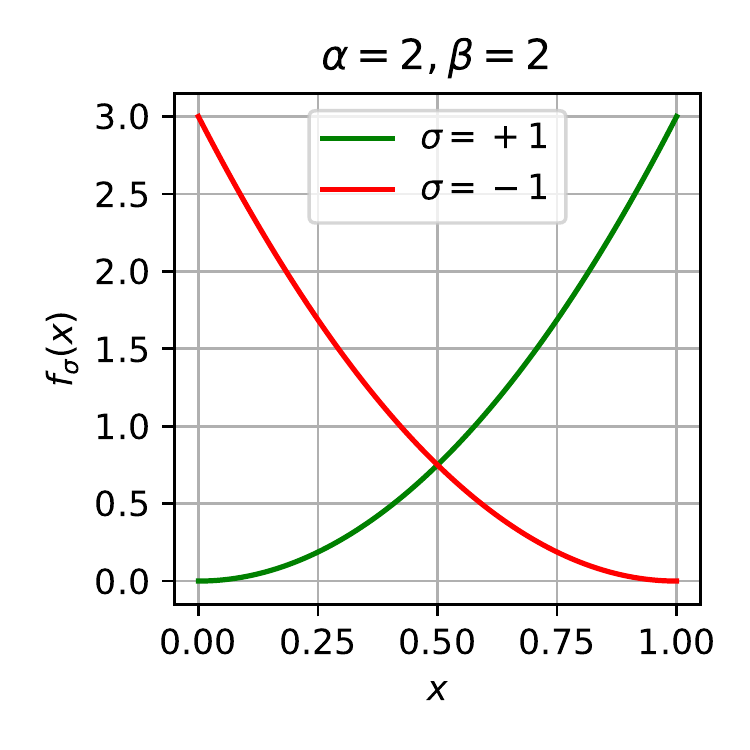}
    \endminipage

    \includegraphics[width=0.5\linewidth]{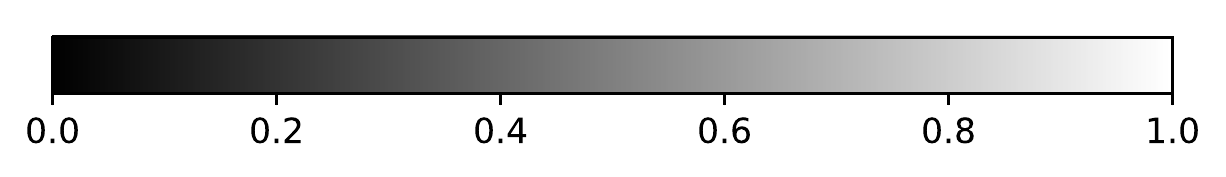}

    \minipage{0.25\textwidth}
	\includegraphics[width=\linewidth]{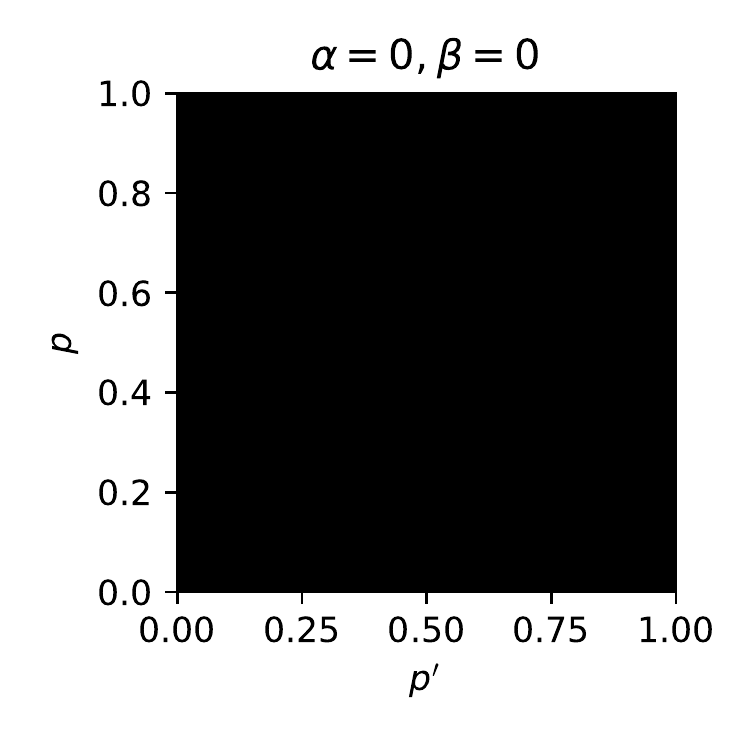}
    \endminipage
    \hfill
    \minipage{0.25\textwidth}
	\includegraphics[width=\linewidth]{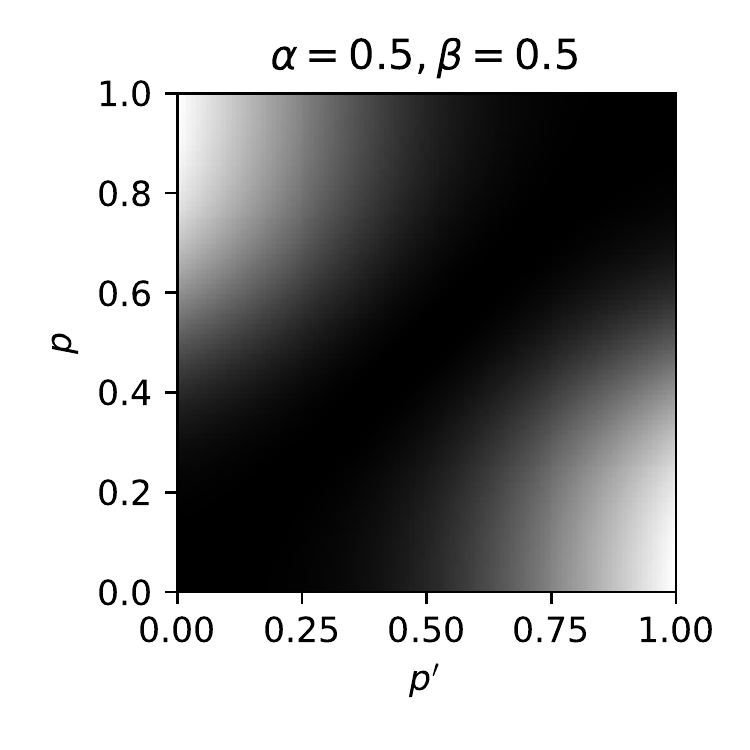}
    \endminipage
    \hfill
    \minipage{0.25\textwidth}
	\includegraphics[width=\linewidth]{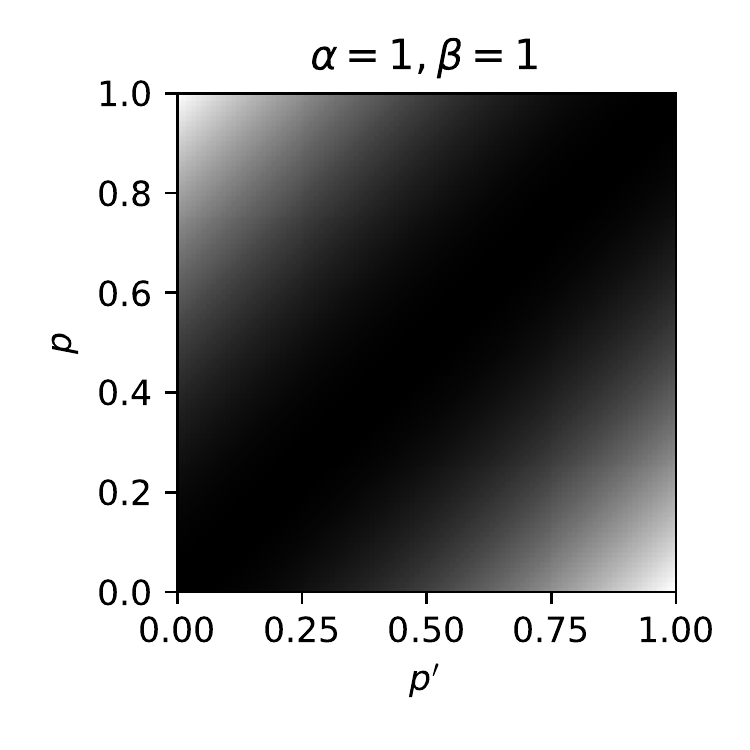}
    \endminipage
    \hfill
    \minipage{0.25\textwidth}
	\includegraphics[width=\linewidth]{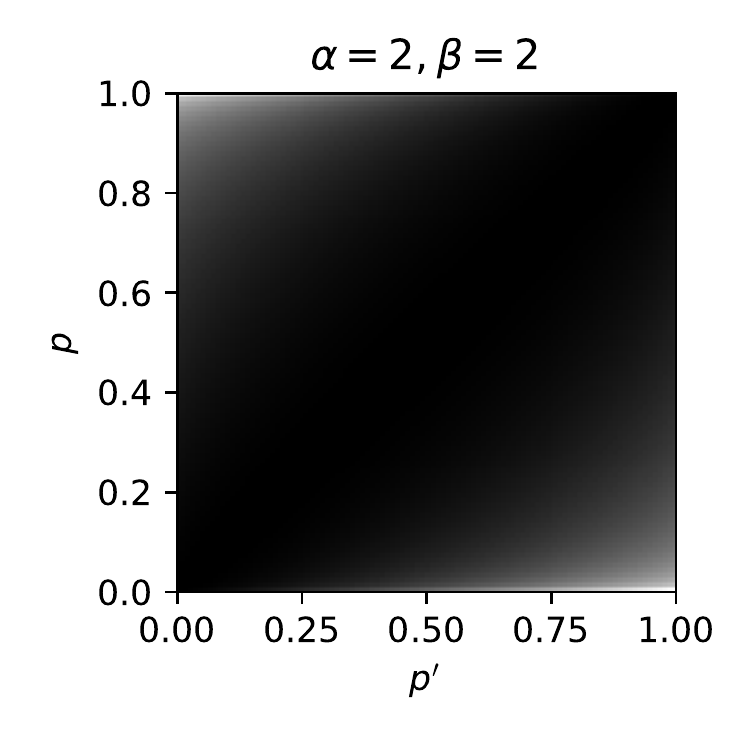}
    \endminipage

    \caption{Pdf's and values of the excess risk $\mathcal{E}(p',p)$ for
    different values of $\alpha, \beta$.}
    \label{fig:plot_sim_exp}
\end{figure}

The results of \cref{fig:plot_sim_exp} show that the optimum for the train 
distribution is significantly different from the optimum for the test distribution
when the problem involves Bayes noise.

\subsection*{On the real data experiment}\label{sec:generalities_expes}

\paragraph{Strategy to induce bias in balanced datasets}

In the real data experiment described in \Cref{sec:num},
a strategy is used to induce class distribution bias or
strata bias, since the data is uniformly distributed on strata for the train and test set.
Since the experiment involves a small test dataset, it is kept
intact, while we discard elements of the train dataset to induce bias between
the train and test datasets. The bias is parameterized by a single parameter
$\gamma$, such that when $\gamma$ is close to one, there is little strata or
class bias, while when $\gamma$ approaches $0$, bias is extreme.

The bias we induce is inspired by a power law, which is often used
to model unequal distributions. 
The distribution on the strata of the train set is modified so that the
generated train set follows a power law.
Formally, the power law distribution $\{p_k'\}_{k=1}^K$ over $S\in\{ 1, \dots, K\}$,
is defined for all $1 \le k \le K$ as
\begin{align*}
    p_k' = \frac{\gamma^{- \frac{\lfloor K/2 \rfloor}{\sigma(k)}} p_k}{ 
    \sum_{l=1}^K \gamma^{- \frac{\lfloor K/2 \rfloor}{\sigma(k)}} p_k },
\end{align*}
where $\sigma$ is a random permutation in $\{1,\dots,K\}$.

To generate a train dataset with modality distribution $\{p_k'\}_{k=1}^K$,
we sample instances from the original train data set $\mathcal{D}_n^\circ=
\{ (X_i', Y_i', S_i') \}_{i=1}^n$, where $Y_i'$ is the class, $S_i'$ is the
strata. The generated train dataset is noted $\mathcal{D}_n$. First, we define
candidates $\mathcal{I}_k = \{ i \mid 1 \le i \le n, S_i' = k \}$ for each
strata $k \in \{1, \dots, K\}$.
Then we select one of the candidate sets $\mathcal{I}_k$ with the probabilities
$p_k'$'s, to remove one of its elements, selected at random, and place it in the
train dataset $\mathcal{D}_n$. We repeat this operation until one of the
candidate sets is empty. A more efficient implementation of this process was used in the
provided code.

\paragraph{Models} We compare two models: a linear model and a multilayer perceptron (MLP) with
one hidden layer of size 1,524.  Given a classification problem of input $x$ of dimension $d$
with $K$ classes, precisely with $d=2048, K=1000$, a linear model simply learns the weights matrix $W \in
\R^{d\times K}$ and the bias vector $b \in \R^K$ and outputs logits 
$l = W^\top x + b$. On the other hand, the MLP has a hidden layer of dimension
$h = \lfloor(d+K)/2 \rfloor$ and learns the weights matrices 
$W_1\in \R^{d, h}, W_2 \in \{ h, K \}$ and bias vectors $b_1 \in \R^h, 
b_2 \in \R^K$ and outputs logits $l = W_2^\top h(W_1^\top x + b_1) + b_2$
where $h$ is the ReLU function, i.e. $h: x \mapsto \max(x,0)$. The MLP model
involves approximatively 5M (million) parameters, while the MLP model uses only
2M. The weight decay or l2 penalization for the linear model 
and MLP model are written, respectively
\begin{align*}
    \mathcal{P} = \frac{1}{2} \lVert W \rVert \quad \text{and} \quad 
    \mathcal{P} = \frac{1}{2} \lVert W_1 \rVert + \frac{1}{2} \lVert W_1 \rVert.
\end{align*}

\paragraph{Cost function}
The cost function is the Softmax Cross-Entropy (SCE), which is the most used
classification loss in deep learning. Specifically, given logits $l = 
(l_1,\dots,l_K) \in \R^K$, the softmax function is $\gamma : \R^k \to [0,1]^K$
with $\gamma = (\gamma_1, \dots, \gamma_K)$ and for all $k \in \{1,\dots,K\}$,
\begin{align*}
    \gamma_k : l \mapsto \frac{\exp(l_k)}{\sum_{j=0}^K \exp(l_j)}.
\end{align*}
Given an instance with logits $l$ and ground truth class value $y$, the
expression of the softmax cross-entropy $c(l,y)$ is
\begin{align*}
    c(l,y) = \sum_{k=1}^K \I\{ y = k \} \log \left(\gamma_k(l)\right).
\end{align*}
The loss that is reweighted depending on the cases as described in
\Cref{sec:main} is this quantity $c(l,y)$. The loss on the test set is never
reweighted, since the test set is the target distribution.  The weights and
bias of the model that yield the logits are tuned using backpropagation on this
loss averaged on random batches of $B$ elements of the training data summed
with the regularization term $\lambda \cdot \mathcal{P}$ where $\lambda$ is a
hyperparameter that controls the strength of the regularization. 

\paragraph{Preprocessing, optimization, parameters} 
The images of ILSVRC were encoded using the implementation of ResNet50 provided
by the library \emph{keras}, see \cite{chollet2015keras}, by taking the
flattened output of the last convolutional layer.

Optimization is performed using a momentum batch gradient descent algorithm on batches
of size 1,000, which updates the parameters $\theta_t$ at timestep $t$ with an
update vector $v_t$ by performing the following operations:
\begin{align*}
    v_t &= \gamma v_{t-1} + \eta\nabla C(\theta_{t-1}), \\
    \theta_{t} &= \theta_{t-1} - v_t,
\end{align*}
where $\eta = 0.001$ is the learning rate and $\gamma=0.9$ is the momentum, as explained in
\cite{DBLP:journals/corr/Ruder16}.
The weight decay parameters $\lambda$ were cross-validated by trying values on the
logarithmic scale $\{10^{-4}, 10^{-3}, 10^{-2}, 10^{-1}, 1\}$ 
and then we tried more fine-grained values between the two best results,
in practice $10^{-3}$ was best and $10^{-2}$ was second best so we tried
$\{ 0.002, 0.003, 0.004, 0.005 \}$. The standard deviation initialization of
the weights $\sigma_0= 0.01$ was chosen by trial-and-error to avoid overflows. 
The learning rate was fixed after trying different values to have fast
convergence while keeping good convergence properties.

\paragraph{Stratified information for ImageNet}

In this section, we detail the data preprocessing necessary to assign strata to
the ILSVRC data. These were constructed using a list of 27 high-level
categories found on the ImageNet website\footnote{
\url{http://www.image-net.org/about-stats}}. 

Each ILSVRC image has a ground truth low level
synset, either from the name of the training instance, or in the validation 
textfile for the validation dataset, that is provided by the ImageNet website.
The ImageNet API \footnote{\url{http://image-net.org/download-API}}
provides the hierarchy of synsets in the form of \emph{is-a} relationships, e.g.
\emph{a flamingo is a bird}. Using this information, for each synset in the 
validation and training database, we gathered all of its ancestors in the hierarchy
that were high-level categories.
Most of the synsets had only one ancestor, which then accounts for one stratum.
Some of the synsets had no ancestors, or even several ancestors in the table,
which creates extra strata, either a \emph{no-category} stratum or a
strata composed of the union of several ancestors. 
The final distribution of the dataset over the created strata is summarized by
\Cref{fig:dist_strata_imagenet}.  Observe the presence of a \emph{no\_strata}
stratum and of unions of two high-level synsets strata, e.g.
\emph{n00015388\_n01905661}.
The definitions of the strata can be requested to the API, 
see \cref{strata_definitions_imagenet} for examples.

\begin{figure}[h]
\centering
\includegraphics[width=\linewidth]{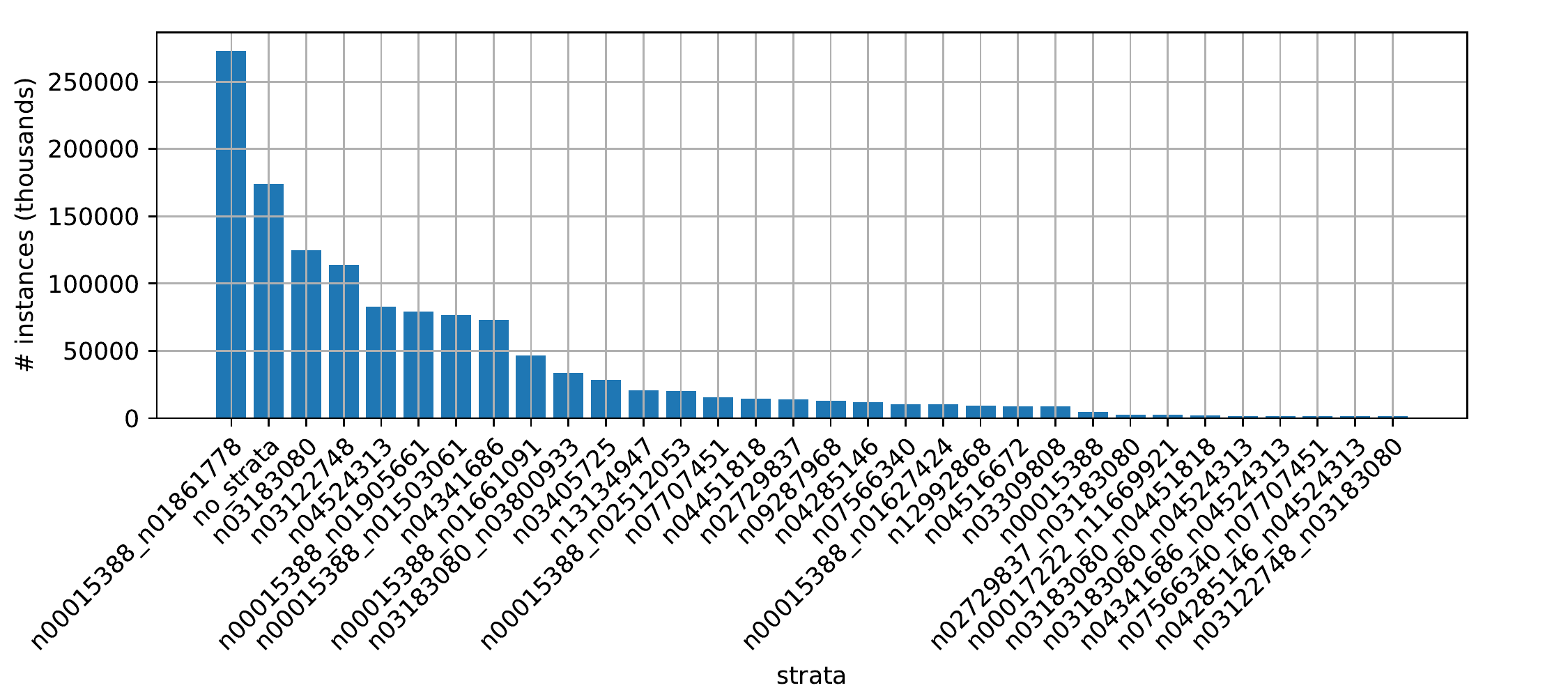}
\caption{Distribution of the ImageNet train dataset over the created strata,
with example definitions in \cref{strata_definitions_imagenet}.}
\label{fig:dist_strata_imagenet}
\end{figure}

\begin{figure}
    \centering
\begin{tabular}[h]{llll}
    \toprule
    \multicolumn{2}{l}{Strata name} & \multicolumn{2}{l}{Definition}\\
	 \cmidrule(lr){1-2}
	 \cmidrule(lr){3-4}
n00015388 & n01861778 & animal, animate being, beast (\dots)& mammal, mammalian \\
\multicolumn{2}{l}{n04524313} & \multicolumn{2}{l}{vehicle} \\
n13134947 & & fruit & \\
n00015388 & n02512053 & animal, animate being, beast (\dots) & fish \\
n00017222 & n11669921 & plant, flora, plant life & flower \\
n07566340 & n07707451 & foodstuff, food product & vegetable, veggie, veg \\
    \bottomrule
\end{tabular}
\caption{Examples of definitions of the strata created for the experiments in
\Cref{sec:num}.}\label{strata_definitions_imagenet}
\end{figure}

\end{document}